\documentclass[letterpaper]{article} 
\usepackage[]{aaai25}  
\usepackage{times}  
\usepackage{helvet}  
\usepackage{courier}  
\usepackage{caption}
\usepackage{subcaption}
\usepackage{multirow}

\usepackage[hyphens]{url}  
\usepackage{graphicx} 
\usepackage{booktabs}
\usepackage{algorithm}
\usepackage{algpseudocode}
\usepackage{subcaption}
\usepackage{bbold}
\usepackage{xspace}
\usepackage{amsfonts}
\usepackage{amsmath}
\usepackage{amsthm}
\usepackage{amssymb}

\newcommand{\method}{\textsc{UaG}\xspace}
\usepackage{natbib}  
\usepackage{caption} 
\theoremstyle{definition}
\newtheorem{definition}{Definition}
\newtheorem{theorem}{Theorem}
\newtheorem{remark}{Remark}
\usepackage{xspace}
\usepackage{todonotes}
\usepackage{enumitem}
\usepackage{xcolor}
\usepackage{algorithm}

\usepackage{newfloat}
\usepackage{listings}
\DeclareCaptionStyle{ruled}{labelfont=normalfont,labelsep=colon,strut=off} 
\floatstyle{ruled}
\newfloat{listing}{tb}{lst}{}
\floatname{listing}{Listing}
\title{Towards Trustworthy Knowledge Graph Reasoning:\\
An Uncertainty Aware Perspective}

\author{
    Bo Ni\textsuperscript{\rm 1}, 
    Yu Wang\textsuperscript{\rm 2},
    Lu Cheng\textsuperscript{\rm 3},
    Erik Blasch\textsuperscript{\rm 4}, 
    Tyler Derr\textsuperscript{\rm 1}
}
\affiliations{
    \textsuperscript{\rm 1}Vanderbilt University, 
    \textsuperscript{\rm 2}University of Oregon, 
    \textsuperscript{\rm 3}University of Illinois Chicago, 
    \textsuperscript{\rm 4}Air Force Research Lab\\

    bo.ni@vanderbilt.edu, yuwang@uoregon.edu, lucheng@uic.edu, erik.blasch.1@us.af.mil, tyler.derr@vanderbilt.edu
}

\usepackage{bibentry}

\begin{document}

\maketitle

\begin{abstract}
Recently, Knowledge Graphs (KGs) have been successfully coupled with Large
Language Models (LLMs) to mitigate their hallucinations and enhance their
reasoning capability, e.g., KG-based retrieval-augmented framework. However, current KG-LLM frameworks lack rigorous
uncertainty estimation, limiting their reliable deployment in high-stakes applications
Directly incorporating uncertainty quantification into KG-LLM frameworks presents challenges due to their 
complex architectures and the intricate interactions between the knowledge graph and language model components. To address this crucial gap,
we propose a new trustworthy KG-LLM framework, \method(\textbf{U}ncertainty \textbf{A}ware Knowledge-\textbf{G}raph Reasoning), which incorporates uncertainty quantification into the KG-LLM framework. We design an
uncertainty-aware multi-step reasoning framework that leverages conformal prediction to
provide a theoretical guarantee on the prediction set. To manage the error rate of the multi-step process, we additionally introduce an error rate control module to adjust the error rate within the individual components. Extensive experiments show that \method can achieve any pre-defined coverage rate while reducing the prediction set/interval size by 40\% on average over the baselines. Our code is available at: \textcolor{blue}{\url{https://github.com/Arstanley/UAG}}
\end{abstract}

%

\section{Introduction}
Large Language Models (LLMs) have recently achieved impressive performance in question-answering tasks due to their unprecedented capability of understanding complex linguistic patterns and generating coherent responses~\cite{openai2024gpt4, huang2022reasoning}. However, LLM's frequent hallucination~\cite{huang2023survey, mundler2024self} and lack of reasoning capability~\cite{lin2021truthfulqa, fu2023mirage} still limit their practical usage when facing domain-specific or complex questions. 

To address the challenges, recently knowledge graphs (KGs) have been coupled with LLMs
to provide factual data and contextual grounding, enhancing the reliability and accuracy of their responses. Knowledge graphs, as a form of structural knowledge representation consisting of factual triplets, provide two unique advantages to overcome the aforementioned challenges of LLMs. First, the additional information retrieved from KGs is directly extracted from reliable external sources, ensuring the currency and faithfulness
of the additional context upon which LLMs generate answers.
Second, KGs are typically represented as triplets where entities are connected by heterogeneous relations, providing explicit logic to enable reasoning and fulfilling complex tasks. To effectively leverage the above two advantages of KGs to enhance LLMs, two representative methods have been developed: 
LLM-based Structural Query and KG-based Retrieval-augmented Generation (KG-RAG). Structural Query relies on LLMs to generate structural queries such as SPARQL~\cite{structgpt} or traverse the KG via beam search~\cite{sun2023thinkongraph, wang2024knowledge}. KG-RAG works by deriving factually correct and context-relevant content from KG to augment the generation of LLMs~\cite{luo2024rog, cok, skg, knowledgeaugmented, yang2023chatgpt, wang2023boosting}. 

\begin{figure}[t]
    \centering
    \includegraphics[width=0.8\columnwidth]{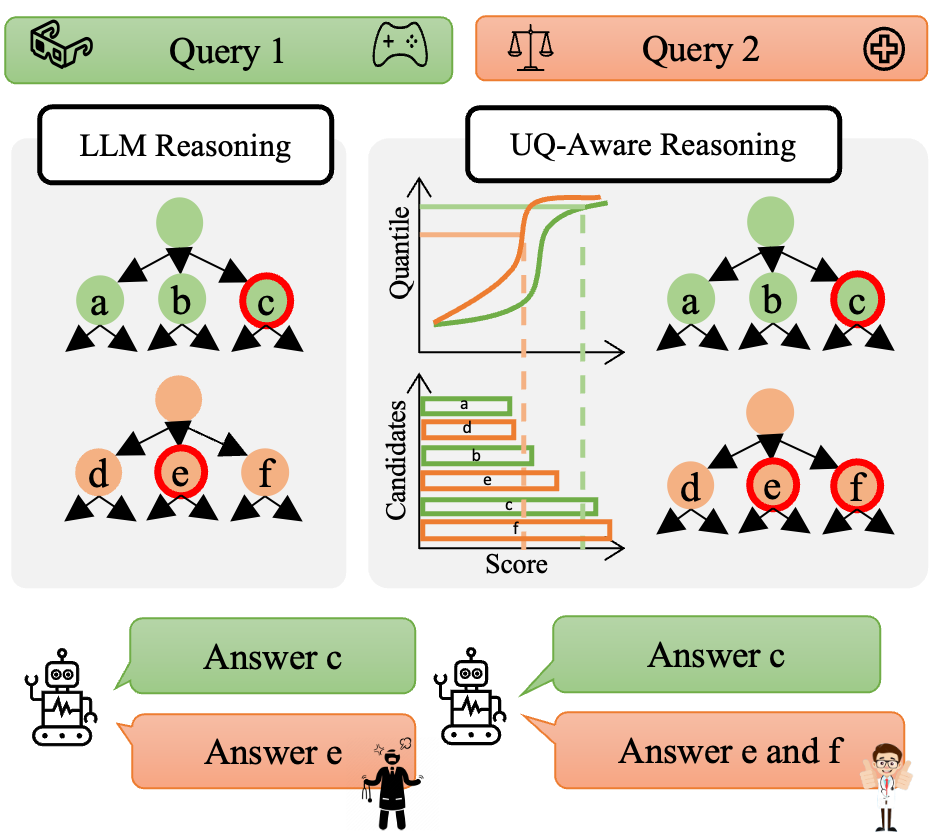} 
    \vskip -1.5ex
    \caption{An illustration of uncertainty quantification in the context of knowledge graph question answering. By quantifying a confidence boundary (dashlines), UQ methods can ensure the accuracy when facing uncertainties.}
    \label{fig:motivation}
     \vskip -2ex
\end{figure}

Despite KG-RAG and Structural Queries' effectiveness in various real-world applications, there
is no systematic investigation into the uncertainty quantification (UQ) aspect
of KG-LLM systems, and therefore, their deployment remains limited in high-stakes
scenarios (e.g., medical analysis, financial decision-making, etc.) where the
cost-of-errors is significant. As illustrated in
Figure \ref{fig:motivation}, without proper uncertainty quantification, existing
approaches fail to generate answers that accurately reflect output confidence.
Recently, several methods have been proposed to model uncertainty for LLMs~\cite{ye2024benchmarking, lin2023generating}. 

 One popular technique for uncertainty quantification is conformal prediction (CP)~\cite{angelopoulos2021gentle}. Given a user-specified error rate $\alpha$, CP produces a prediction set that guarantees a coverage rate $1-\alpha$ by calibrating the model prediction on a hold-out calibration set. Despite the model agnostic, distribution free characteristics of CP, directly applying it to LLMs is challenging because of the unbounded output space. To address it, ~\citeauthor{ye2024benchmarking} and ~\citeauthor{kumar2023conformal} used the logits generated by the contemporary causal language models as the prediction probability; \citeauthor{su2024api} and \citeauthor{lin2023generating} considered the commercial black-box language models and sampled the responses to approximate the confidence; Most recently, \citeauthor{quach2023conformal} extended the general risk control framework to enable conformal prediction in large language models~\cite{quach2023conformal}. 
 
Despite the numerous methods that have utilized conformal prediction (CP) to model the uncertainty of large language models (LLMs), no previous work has focused on equipping knowledge graph-language models (KG-LLMs) with uncertainty quantification. This task presents several unique challenges. First, the complexity of knowledge graphs (KGs), characterized by multiple hops connecting queries to solutions, makes it difficult to accurately calibrate the graph reasoning process. Second, the high-dimensional output distribution of the underlying language models further complicates the calibration process. In KG-LLM applications, there are often multiple valid solutions to a single query, making the coverage of these solutions critically important. Thus, the uncontrolled high-dimensional output poses a significant challenge because it complicates the process of identifying and validating the correct solutions among many possibilities. This lack of control can lead to the propagation of errors and reduced reliability in the model's predictions. To address these challenges, novel techniques specifically tailored for knowledge graph question answering (KGQA) need to be developed.

Therefore, to fill this gap, in this paper, we introduce \method(\textbf{U}ncertainty \textbf{A}ware Knowledge-\textbf{G}raph Reasoning), a novel uncertainty-aware knowledge graph reasoning framework that takes advantage of both the structural knowledge representation and the reasoning capability of LLM. Specifically, instead of relying on LLM to
generate the answers from the open domain, we retrieve an initial set of
uncertainty aware answers through beam searching the knowledge graph. We guide
the beam search process with conformal prediction to achieve the theoretically guaranteed
coverage rate. 
In order to facilitate more faithful reasoning in LLM, we also
retrieve reasoning paths through a \textit{planning-retrieval} module~\cite{luo2024rog}. Then, we leverage the LLM to generate the answer with the retrieved reasoning paths. To combine the power of generation and the fidelity of the retrieved candidates, we additionally calibrate the similarity measurements between the retrieved candidates and the generated answers. One challenge of the multi-step framework is the propagation of the error rates. In order to adjust the error rate aggregation across multiple components, we leverage the Learn Then Test (LTT) framework~\cite{angelopoulos2021learn} to control the error rate in the individual components.   

We perform experiments on two widely
used multi-hop knowledge graph QA datasets and demonstrate that \method is able to
satisfy the uncertainty constraint while maintaining a reasonable size of
prediction.

In summary, our contributions are summarized as follows: 
\begin{itemize}
    \item We propose a novel framework, \method, for uncertainty quantification with LLM-based KGQA, addressing a significant gap in the existing literature.
    \item We extend the \textit{learn-then-test} paradigm to KGQA by modeling the distribution of the error propagation, improving the reliability of the model outputs.
    \item Extensive experiments demonstrate the effectiveness of our framework on two multi-hop knowledge graph QA benchmarks, outperforming baseline methods in terms of both coverage and robustness under uncertainty. 
\end{itemize}

The rest of the paper is organized as follows: Section \ref{section:prelim} provides the preliminaries; Section \ref{section:method} introduces our \method in detail; experiments and ablation studies in Sections \ref{section:experiment} and \ref{section:ablation}; related work in Section \ref{section:related}; lastly, conclusions and future work are presented in Sections \ref{section:conclusion}.

\section{Preliminiaries}
\label{section:prelim}
\subsection{Conformal Prediction} \label{sec-cp}

Conformal prediction (CP) is a distribution-free model-agnostic approach for uncertainty quantification. By calibrating the model prediction on a held-out calibration set, CP produces sets of predicted intervals that
contain the ground-truth labels with a user-specified error rate
$\alpha$~\cite{angelopoulos2021gentle, shafer2008tutorial}. 

Without loss of generality, considering the $i^{\text{th}}$ sample $(x_i, y_i)$ with $x_i$ being the input feature and $y_i$ being the corresponding ground-truth output, we denote the calibration set as $\mathcal{D}^{\text{cal}} = \{(x_i, y_i)\}_{i = 1}^{n}$. Conformal Prediction (CP) guarantees the error rate of the prediction set with the following steps.
\begin{itemize}[leftmargin=*]
    \item \textbf{Define non-conformal score:} A heuristic-based uncertainty estimation function $S: \mathcal{X}\times \mathcal{Y} \rightarrow \mathbb{R}$ is defined to calculate the non-conformal score $s_i=S(x_i, y_i)$, which measures the uncertainty of the prediction $y_i$ given the input $x_i$ with larger score indicating worse agreement between $x_i$ and $y_i$. 
    \item \textbf{Compute the quantile:} CP then calculates the non-conformal score for every input-output pair in the calibration set to obtain $\mathcal{S}^{\text{cal}} = \{s_i\}_{i = 1}^{n}$ and further compute the conformal score $q_{\alpha}^{S, \mathcal{D}_{\text{cal}}}$ as the $\frac{\lceil(n+1)(1-\alpha)\rceil}{n}$ quantile of the calibration scores $\mathcal{S}^{\text{cal}}$:
    \begin{equation}
        q_{\alpha}^{S, \mathcal{D}_{\text{cal}}} = \text{Quant}(\{S(x, y)|(x, y)\in \mathcal{D}_{\text{cal}}\}, \frac{\lceil(n+1)(1-\alpha)\rceil}{n})
    \end{equation}
    \item \textbf{Prediction set:} the final prediction set satisfying the $\alpha$ error rate with respect to the calibration set $\mathcal{D}_{\text{cal}}$ would be $C(X_\text{test}) = \{y|y \in \mathcal{Y}, S(X_{\text{test}}, y) \leq q_{\alpha}^{\mathcal{S}, \mathcal{D}_{\text{cal}}}$\}.
\end{itemize}

\subsection{Learn Then Test (LTT)} \label{section:ltt}
Recently, the Learn Then Test (LTT) framework~\cite{angelopoulos2021learn} was introduced that extends conformal prediction to manage the expectation of any loss functions. 
Specifically, it is achieved by approaching the hyper-parameter selection as a multiple-hypothesis testing problem. Next, we formally introduce the LTT framework. 

Formally, let $L_\lambda: \mathcal{Y} \times \hat{\mathcal{Y}} \rightarrow \mathbb{R}$ be any loss function with a hyperparameter configuration $\lambda \in \Lambda$. Notably, $\lambda$ could be multi-dimensional. Let $\alpha \in \mathbb{R}$ be the user-defined error rate for $L_\lambda$ (i.e., $\mathbb{E}(L_\lambda) \leq \alpha$). 
Using calibration set $\mathcal{D}^{\text{cal}}$, LTT computes a set of valid configurations $\Lambda_{\text{valid}} \in \Lambda$ satisfying 
\begin{equation}
\label{eq:ltt}
    \mathbb{P}\bigg( \underset{\lambda \in \Lambda_{\text{valid}}}{\text{sup}} \mathbb{E}[L_\lambda | \mathcal{D}_{\text{cal}}]\leq \alpha \bigg) \geq 1 - \delta
\end{equation}
where $\delta$ represents the desired confidence level on the selection of configurations. Intuitively, $\delta$ is the probability that the valid configurations we identify will truly meet our error guarantee. To calibrate the configurations on $\mathcal{D}_{\text{cal}}$, we define the null hypothesis $H_0^\lambda: \mathbb{E}[L_\lambda] > \alpha$ for each $\lambda \in \Lambda$ and calculate a super-uniform p-value $p_\lambda$ using concentration inequalities. Then we can leverage any family-wise error rate (FWER) controlling algorithms to identify the non-rejected configurations $\Lambda_{\text{valid}}$.

\begin{theorem}[Learn Then Test~\cite{angelopoulos2021learn}] 
Suppose $p_\lambda$ is super-uniform under $H_0^\lambda \forall \lambda \in \Lambda$. Let $\mathcal{T}$ be any FWER-controlling algorithm at level $\delta$. Then $\Lambda_{valid}$ satisfies Eq.~\eqref{eq:ltt}. 
\end{theorem}
\subsection{Problem Definition} Given a knowledge graph
$\mathcal{G}=(\mathcal{E}, \mathcal{R})$, where $\mathcal{E}$ is the set of entities and $\mathcal{R}$ is the set of relations. The edges in the knowledge graph are represented as triplets (i.e., $(e_s, r, e_o) \in \mathcal{G}$) with $e_s \in \mathcal{E}$ being the head entity, $r \in \mathcal{R}$ being the relation, and $e_o \in \mathcal{E}$ being the tail entity. For the rest of this paper, without further specification, we denote the calligraphic font as sets (such as $\mathcal{X}, \mathcal{Y}$) and capital letters as functions (such as $S$, $F$, for score functions and loss functions, respectively).  

\begin{definition}[Uncertainty quantification for multi-hop knowledge graph question answering] Given a question $q$ in natural language with the ground-truth answer set being $\mathcal{Y}_{\text{test}}$, assuming the user specified error rates as $\alpha$ and $\delta$, the task here is to derive a algorithm $C_{\lambda}$ hyper-parametrized by $\lambda$, which takes the input $X_{\text{test}}$ and predicts the output set $\hat{\mathcal{Y}}_{\text{test}}$ that satisfies:
\begin{equation}\label{eq:problem} 
\mathbb{P}\bigg(\mathbb{P}\left(\hat{e} \in \mathcal{Y}_\text{test}, \forall \hat{e} \in C_{\lambda}(X_{\text{test}}) \mid \mathcal{D}_{\text{cal}}\right) \geq 1-\alpha\bigg) \geq 1-\delta     
\end{equation}

\end{definition}

\begin{remark} The error rate $\delta$ controls the probability that the inner probabilistic guarantee holds. Specifically, it ensures that with at least $(1 - \delta)$ confidence, the prediction set $C_{\lambda}(X_{\text{test}})$ will contain the correct answers with a probability of at least $(1 - \alpha)$.

    \textbf{Intuition:} Consider $\alpha$ as a measure of how tolerant we are to errors in individual predictions. The parameter $\delta$, on the other hand,
    controls our overall confidence in this process. 
\end{remark}

\section{Method}
\label{section:method}
\begin{figure}[t]
    \centering
    \includegraphics[width=0.95\columnwidth]{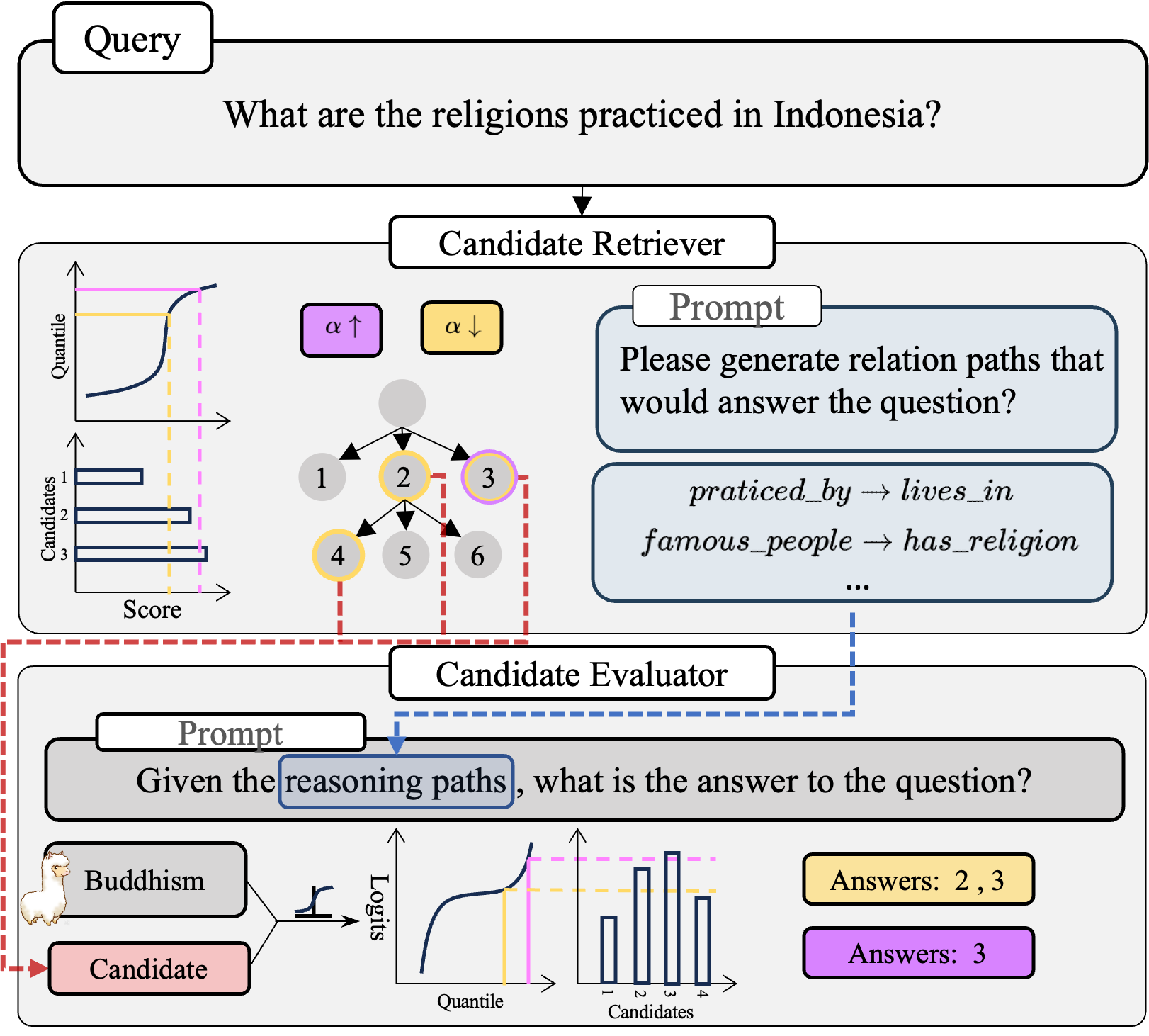}
    \caption{An illustration of our \method framework.}
    \label{fig:framework}
\end{figure}

\method is a multi-hop knowledge graph reasoning framework that incorporates
uncertainty quantification (UQ) into its reasoning process. It consists of three
components: 
UQ-aware
candidate retriever, UQ-aware candidate evaluator, and Global Error Rate
Controller, as shown in Figure~\ref{fig:framework}. Given a
user-specified query, the \textbf{UQ-aware Candidate Retriever} selectively
retrieves neighboring nodes and formulates reasoning paths to ensure a
pre-defined error rate. The \textbf{UQ-aware Candidate Evaluator} then reasons over the retrieved candidates and reasoning paths to produce the final answer set. Subsequently, to avoid overly conservative predictions, the
\textbf{Global Error Rate Controller} adjusts individual error rates by
calibrating them using the LTT framework introduced in 
the previous section.

\subsection{UQ-aware Candidate Retriever}
The first component of our \method framework is the candidate retriever, which retrieves candidate entities from the knowledge graph based on the question or current paths. Previous research shows that multi-hop graph traversal frameworks improve LLM reasoning and answer faithfulness~\cite{wang2024knowledge, sun2023thinkongraph}.  However, these frameworks typically use heuristics like Top-$K$ candidate selection based on textual similarity, which lack a theoretical basis and cannot ensure a consistent error rate, making them unreliable for high-stakes scenarios. To address this, we incorporate uncertainty quantification into graph traversal, replacing heuristic selection with error-bounded selection. Specifically, we apply conform prediction to two steps: retrieving candidate paths and candidate neighbors.

\subsubsection{Retrieving Candidate Path}

Formally, for each traversed path currently ending at node $v$, we update the breadth-first-search tree queue by following:
\vspace{-1ex}
\begin{equation}
    \label{eq:visiting}
    \Bigl\{s ~\big|~ s \in \mathcal{N}(v), ~S_1\Bigl(Q||(||_{i=0}^{j-1} r_i), r_j\Bigr) < q_{\alpha_1}^{S_1, \mathcal{D}_{\text{cal}}}\Bigr\}
\end{equation}
where \(\mathcal{N}(v)\) is the neighborhood set of $v$, $Q||(||_{i = 0}^{j - 1}r_i)$ are the concatenated relations traversed so far in the path with starting passages being $Q$, $r_j$ is the current relation between the entities $v$ and $s$, and $S_1$ is the uncertainty estimation score for the candidate retriever. In our context, we choose $S_1$ to be the textual similarity between two textual passages.

\subsubsection{Retrieving Candidate Neighbor}
While traversing the graph, we also want to determine if the current node should be added into the candidate set. As a result, we add the following set to the prediction set while visiting node $v$:
\vspace{-1ex}
\begin{equation}
    \label{eq:ans}
    \Bigl\{s ~\big| s \in \mathcal{N}(v), S_1\Bigl(Q, ||_{i=0}^{j} r_i\Bigr)<q_{\alpha_2}^{S_1, \mathcal{D}_{\text{cal}}} \Bigr\}
\end{equation}

Eq.~\eqref{eq:visiting} differs from Eq.~\eqref{eq:ans} in their focus during
the traversal process. Eq.~\eqref{eq:visiting} focuses on identifying the likely answer paths by assessing the similarity between the traversed path and the current relation. In contrast, Eq.~\eqref{eq:ans} evaluates if the entire path up to the current relation sufficiently answers the query by comparing it to the initial query. This distinction allows for dynamic adjustment of the traversal process and candidate set based on the evolving query context and encountered relations.

\subsection{UQ-aware Candidate Evaluator}
The second component of \method is an uncertainty-aware candidate evaluator based on large language models (LLMs). Directly prompting the language model has demonstrated significant advantages in knowledge graph question answering (KGQA) tasks, particularly due to its ability to generate diverse and contextually relevant answers. However, this open-ended generation approach lacks any theoretical guarantees on the correctness of the answers, which can lead to unreliable outputs. Conversely, the retrieved candidates are constrained by a predefined error rate, ensuring the reliability of the prediction set. However, this constraint often results in overestimating the prediction set size, encompassing a broader range of potential answers than necessary.

Thus, to harness the strengths of both approaches, we introduce a calibration process that optimizes the balance between them. By aligning the similarity between the LLM-generated candidates and the retrieved set, we aim to refine the prediction set size while maintaining a bounded error rate. This ensures that the final output is not only theoretically sound but also practically effective.

To calculate the final answer set, we define the non-conformal score for the
evaluator by taking the similarity between the retrieved answers and the generated answers. Formally, given a set of candidates $\mathcal{C}$ and retrieved
reasoning paths $\mathcal{P}$, let $\Phi$ be the LLM generation function. Then, the final answer set is defined as follows:

\begin{equation}
    \label{eq:llm}
    \Bigl\{ a \in \mathcal{C} \mid S_1\big(a, \Phi(\mathcal{P})\big) < q_{\alpha_{\text{llm}}}^{S_1, \mathcal{D}_{\text{cal}}} \Bigr\}
\end{equation}

\subsection{Global Error Rate Controller}
Since \method involves multiple components, directly applying the user-defined error rate $\alpha$ to Eqs.~\eqref{eq:visiting}-\eqref{eq:llm} may not achieve the desired error rate. Errors from the knowledge graph traversal can propagate to the inference stage, potentially exceeding the user-specified tolerance unless errors across components are perfectly correlated.

To address this issue, we leverage the LTT framework (Section 2)
to find the best error rate for each component by treating their error rates as hyper parameters. Let $\lambda = (\alpha_1, \alpha_2, \alpha_3) \in (0, 1]^3$ be the individual error rates for each component. The finite search space \(\Lambda\) can be represented as the Cartesian product of these sets for \(\alpha_1\), \(\alpha_2\), and \(\alpha_3\):
\[
\Lambda = \{(\alpha_1, \alpha_2, \alpha_3) \mid \alpha_1, \alpha_2, \alpha_3 \in \{h, 2h, \ldots, 1\} \}
\]
where $h$ is the hyper parameter to control the size of the search space, bounded by 0 and 1. For each $\lambda \in \Lambda$, we compute a valid p-value $p_\lambda$ for the null hypothesis $H_0^\lambda: \mathbb{E}[L_\lambda] > \alpha$ by computing the concentration inequality on the calibration dataset $D_{cal}$.

\begin{theorem}[Binomial tail bound p-values~\cite{quach2023conformal}]
Let $\text{Binom}(n, \alpha)$ denote a binomial random variable with sample size $n$ and success probability $\alpha$. Then 
$
    p_\lambda = \mathbb{P}(\text{Binom}(n, \alpha) \leq 
    \sum_{D_{cal}} L_\lambda
$ 
is a valid p-value for $H_0^\lambda: \mathbb{E}[L_\lambda] > \alpha$. 
\end{theorem}
In our context, we define the loss function $F$ as the error rate in the
system. Let $\mathcal{T}$ be a FWER-controlling algorithm\footnote{FWER-controlling algorithms take a set of null hypothesis and ensures the type-I error rate is controlled by a pre-defined value $\delta$.} $\mathcal{T}:
(\mathcal{P}, \mathbb{R}) \rightarrow \mathcal{P}'$ where $\mathcal{P}$ denotes
a family of p-values that we want to control based on a given error rate. Then
we can identify a set of configurations $\Lambda_{valid} \subseteq \Lambda$ that
satisfies Eq.~\eqref{eq:ltt} by taking $\mathcal{T}(\{p_\lambda \mid \lambda \in \Lambda\}, \delta)$ as $\Lambda_{valid}$.
\begin{theorem}[Coverage Guarantee of \method] 
\label{thm:guarentee}
If $\lambda \in \Lambda_{valid}$, then Eq. \eqref{eq:problem} is satisfied.
\end{theorem}
\begin{proof}
From Eq. \eqref{eq:ltt}, $\forall \lambda \in \Lambda_{valid}$,
\[
\mathbb{P}\left(E[L_\lambda \mid D_{\text{cal}}] \leq \alpha\right) \geq 1-\delta
\]
In the context of KGQA,  
\begin{align*}
    L_\lambda = \mathbb{1}\{\nexists \hat{e} \in C_\lambda(X): \hat{e} \in \mathcal{E}_{ans}\}
\end{align*}
Note that,
\begin{align*}
    &E[L_\lambda \mid D_{\text{cal}}] = E[\mathbb{1}\{\nexists \hat{e} \in C_\lambda(X)\text{: } \hat{e} \in \mathcal{E}_{ans}\} | D_{cal}] \\ 
    &= \mathbb{P}\left(\nexists \hat{e} \in C_\lambda(X)\text{: } \hat{e} \in \mathcal{E}_{ans} | D_{cal} \right)
\end{align*}
Substitute $E[L_\lambda \mid D_{\text{cal}}]$ in Eq. \eqref{eq:ltt}, then we have
\begin{align*}
    \mathbb{P}\left(\mathbb{P}\left(\nexists \hat{e} \in C_\lambda(X)\text{: } \hat{e} \in \mathcal{E}_{ans} | D_{cal} \right) \leq \alpha\right) \geq 1-\delta 
\end{align*}

This is equivalent to
\begin{equation}
    \mathbb{P}\left(
    \mathbb{P}\left(\hat{e} \in \mathcal{E}_{\text{ans}}, \forall \hat{e} \in C_\lambda(X_{\text{test}}) \mid D_{\text{cal}}\right) \geq 1-\alpha\right) \geq 1-\delta
\end{equation}
\end{proof}

In addition, to minimize the returned set size to make the prediction most useful to the users, we use the validation set to select the configuration $\lambda \in \Lambda_{valid}$ with the smallest average set size.

\begin{figure*}[h!]
        {%
            \includegraphics[width=.29\linewidth]{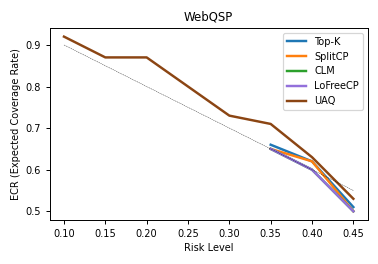}%
            \label{subfig:a}%
        }\hfill
        {%
            \includegraphics[width=.29\linewidth]{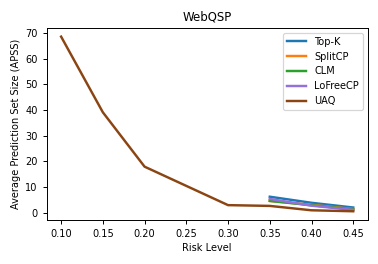}%
            \label{subfig:b}%
        }\hfill
        {%
            \includegraphics[width=.29\linewidth]{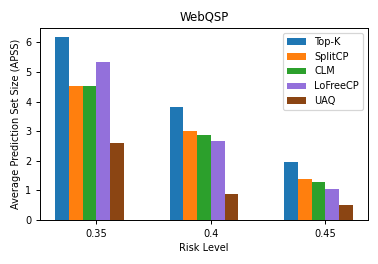}%
            \label{subfig:c}%
        }\\
        \subfloat[Risk Level vs. ECR]{%
            \includegraphics[width=.29\linewidth]{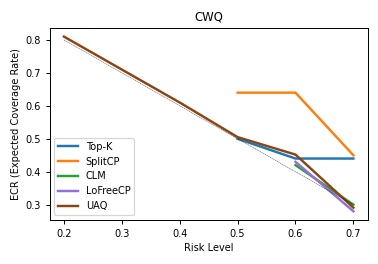}%
            \label{subfig:d}%
        }\hfill
        \subfloat[Risk Level vs. APSS]{%
            \includegraphics[width=.29\linewidth]{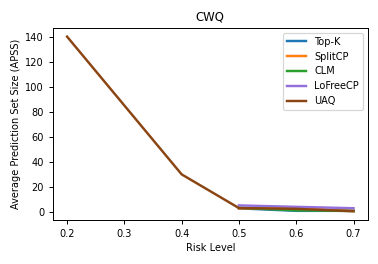}%
            \label{subfig:a}%
        }\hfill
        \subfloat[APSS]{%
            \includegraphics[width=.29\linewidth]{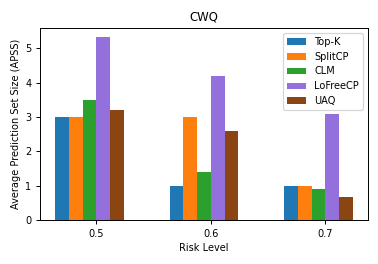}%
            \label{subfig:b}%
        }\\
        \caption{Uncertainty quantification results for \method and the baselines. } 
        \label{fig:res}
    \end{figure*}

\begin{table}[t]
    \centering
    \small 
    \setlength\tabcolsep{2.5pt}
    \caption{Statistics of KGQA datasets.}
    \vskip -0.75ex
    \label{tab:dataset}
    \begin{tabular}{@{}c|ccccc@{}}
        \toprule
        & & & \multicolumn{3}{c}{Extracted KG Subgraphs }\\
        
        Datasets & \#Train & \#Test & $\substack{Max \\\#hops}$ &  $\substack{Avg \\ \#Nodes}$ & $\substack{Avg \\ \#Edges}$ \\ \midrule
        WebQSP & 2,826 & 1,628 & 2 & 1,374 & 2,909 \\
        CWQ & 27,639 & 3,531 & 4 & 1,256 & 2,615 \\
        \bottomrule
    \end{tabular}
\end{table}

\section{Experiment}
\label{section:experiment}
In this section, we conduct extensive experiments on two widely used multi-hop KGQA datasets~\cite{luo2024rog, sun2023thinkongraph, complexKBQA}. Our experiments are designed to rigorously evaluate the uncertainty quantification performance of \method, employing standard UQ metrics to assess its effectiveness against state-of-the-art baselines. 

\subsection{Datasets}
We evaluate \method with two widely used benchmark dataset for KGQA: WebQuestionSP (WebQSP)~\cite{yih2016value} and Complex WebQuestions (CWQ)~\cite{talmor2018web}. WebQSP contains up to 2 hop questions, and CWQ contains up to 4 hop questions. Additionally, Freebase~\cite{freebase} is used as the underlying knowledge graph for both datasets. The dataset statistics are shown in Table \ref{tab:dataset}. Note that for calibration, we use the training partition. 

\subsection{Evaluation Metrics}
Compared to traditional KGQA tasks that typically rely on evaluation metrics such as hits@1~\cite{luo2024rog}, our approach focuses on measuring the uncertainty quantification of the methods. To this end, we use 
empirical coverage rate and average prediction set size as the key metrics to  
measure the accuracy and efficiency of the models, respectively~\cite{angelopoulos2021gentle}. We aim to provide insights into how well the various methods balance 
accuracy with the efficiency of its predictions. The details of these metrics are the following: 

\begin{itemize}

\item \textbf{ECR (Empirical Coverage Rate)}: 
This measures how well the uncertainty quantification model satisfies the error rate.  
For knowledge graphs, we consider the accuracy of the prediction set.

\item 
\textbf{APSS (Average Prediction Set Size)}: 

This  
evaluates the effectiveness of the uncertainty quantification model, with smaller average prediction set sizes indicating greater efficiency in selecting  
the most likely answers. 

\end{itemize}

\subsection{Baselines}
We compare \method with the following uncertainty quantification baselines.
Because there has been no existing work considering uncertainty quantification
in the KG-LLM paradigm, we include the previously proposed existing LLM-based uncertainty quantification methods as baselines.
\begin{itemize}
    \item \textbf{Top-K}: Non-CP method without coverage guarantee. It includes responses with the $K$ highest probabilities for each question in the test set~\cite{huang2023conformal}. 
    \item \textbf{Standard Split Conformal Prediction (SplitCP}). Standard conformal prediction~\cite{shafer2008tutorial} where we follow the framework outlined in previous work for its application on language models~\cite{ye2024benchmarking}.
    \item \textbf{Conformal Language Modeling (CLM)} A logit-based CP method that utilizes the general risk control framework directly on the output of the large language models~\cite{quach2023conformal}.
    \item \textbf{Logit-free Conformal Prediction for LLMs (LoFreeCP)}. 
     A 
     state-of-the-art CP method that designs specific techniques for handling low-frequency events in language models~\cite{su2024api}.
    
\end{itemize}

\subsection{Implementation Details}
For the implementation, we set $\delta$ to be 0.05, and use Llama3-8b as our backbone large language model. For the encoder $g$, we use the SentenceTransformer model and pre-train it on our training data. 
Other implementation details are further documented in the appendix. 

\subsection{Primary Results}

The experimental results are presented in Figure~\ref{fig:res}. For clarity and conciseness, we provide the most relevant graphs in the main paper and include the detailed tabular results in the supplementary material. In Figure~\ref{fig:res}(a), we first plot the Empirical Coverage Rate (ECR) against the risk level. The region above the diagonal dashed line, represents the acceptable coverage rate given the specified risk tolerance level. As anticipated, all CP-based methods satisfy the error guarantee within this boundary. However, we observe that the baseline methods fail to achieve error rates below 0.35, as their generation components alone 
cannot produce results better than this threshold. In contrast, our proposed \method, 
which integrates both retrieval and generation, consistently achieves the desired error rates across a wider range of risk levels.

In Figure~\ref{fig:res}(b) (i.e., the second column), we plot the prediction set size against the risk level. As expected, with an increase in risk level, there is a corresponding decrease in prediction set size, reflecting the trade-off between risk and prediction confidence. This trend is observed as the model becomes less conservative at higher risk levels, thereby reducing the number of elements in the prediction set.

Finally, we compare the zoomed-in view of prediction set sizes across different methods in Figure~\ref{fig:res}(c). Despite the performance at the 0.6 level in CWQ, our proposed \method consistently outperforms or performs competitively with the baseline methods, demonstrating its ability to maintain high accuracy while reducing uncertainty in predictions. This advantage can be attributed to the expressive power of the traversed reasoning paths, which effectively guide the retrieval process, as well as the robust combination of retrieval and generation that enhances the precision of the prediction set.

\subsection{Case Study}
We demonstrate the real-life implications and usefulness of \method through a detailed case study. Specifically, we select an instance among the samples where existing KGQA methods produce mispredicted or incomplete answers and analyze how \method overcomes these challenges. 

As shown in Table \ref{tab:case}, the ground truth to the question \textit{What are the religions practiced in Indonesia?} includes four distinct answers, showcasing the diversity of religious practices in the country. However, directly prompting the question to the language models results in only a single answer—Islam—even when using the state-of-the-art KGQA model. This outcome poses a significant threat to fairness and may introduce bias, as it fails to capture the full spectrum of religious diversity in Indonesia. \method addresses this challenge by incorporating risk level control, which allows for a more comprehensive and balanced prediction.

The prediction set is constructed by setting the risk tolerance level to 0.2. While the resulting set includes several additional related answers—The Religion of Java, a reference to a popular book by anthropologist Clifford Geertz focused on the religion of Java (an island of Indonesia), and The Language of the Gods, an alias for Sanskrit, the ancient classical language of Hinduism—it successfully covers all four ground truth answers. This ensures a more accurate and inclusive representation of the answer to the given question, thereby reducing the risk of bias and enhancing the trustworthiness of the predictions.

\begin{table}[t]
\centering
\small 
\setlength\tabcolsep{1.5pt}
\caption{Case Study on Sample Question from WebQSP.}
\vskip -0.5ex
\begin{tabular}{@{}ll@{}}
\toprule
Question: & \textit{What are the religions practiced in Indonesia?} \\ \midrule
Ground Truth: & \textit{Catholicism}, \textit{Hinduism}, 
                       \textit{Protestantism}, \textit{Islam} 
                      \\ \toprule \\
Predicted Answers: & \\ \midrule
~~~RoG{\scriptsize ~\cite{luo2024rog}} & \textbf{Islam} \\\midrule
~~~\method           & \textbf{Catholicism}, \textbf{Hinduism},  \textbf{Protestantism}, \\ & The Religion of Java,
                       \textbf{Islam}, \\ & The Language of the Gods\\ 
 \bottomrule
\end{tabular}
\label{tab:case}
\vskip 1.2ex
\end{table}

\section{Ablation Studies}
\label{section:ablation}
We conduct ablation studies on \method to comprehensively evaluate its effectiveness and robustness. Specifically, we first examine how \method compares to state-of-the-art KGQA methods, as robust performance on standard KGQA tasks is critical in daily uses. Additionally, we investigate the effects of pre-training and different similarity measures on the model performance.

\begin{table}[h]
\caption{KGQA Results on WebQSP. }
\vskip -0.5ex
\label{table:rog}
\centering
\small 
\begin{tabular}{@{}llc@{}}
\toprule
\textbf{Category} & \textbf{Method} & \textbf{Hits@1} \\ \midrule
\multirow{3}{*}{\textit{Fine-tuned}} 
  & TIARA~\cite{shu2022tiara} & 75.2 \\
  & DeCAF~\cite{yu2023decaf} & 82.1 \\
  & RoG~\cite{luo2024rog} & \textbf{85.7} \\
   \midrule
\multirow{5}{*}{\textit{Prompting}} 
  & COT~\cite{wei2022chainofthought} & 39.1 \\ 
  & ToG~\cite{sun2023thinkongraph} & 57.4\\
  & KD-CoT~\cite{wang2023knowledgedriven} & 63.7 \\
  & StructGPT~\cite{structgpt} & 72.6 \\
  & \method-Top-1 & 66.8 \\ 
  & \method-Top-3 & \textbf{73.4} \\ \midrule
\end{tabular}
\vskip -1ex
\end{table}

\subsection{Comparing to KGQA Methods}
Here we compare our method to state of the art KGQA methods~\cite{luo2024rog} by selecting the top responses from \method given a fixed error rate $\alpha$. We select $\alpha=0.2$ for the purpose of testing. Our result is shown in Table \ref{table:rog}. Overall, the table 
highlights that our model attains performance on par with current state-of-the-art for the traditional benchmark KGQA tasks (that utilize Hits@1 for evaluation). This underscores the robust reasoning capabilities of the \method and the efficiency of the candidate retrieval process.

\begin{figure}[t!]
    \centering
    \begin{subfigure}[t]{0.5\columnwidth}
        \centering
        \includegraphics[height=1.2in]{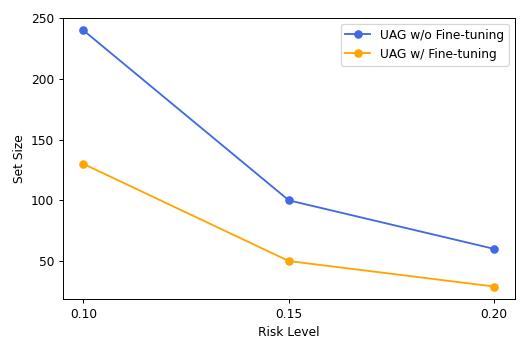}
        \caption{Pre-Training}
    \end{subfigure}%
    ~ 
    \begin{subfigure}[t]{0.5\columnwidth}
        \centering
        \includegraphics[height=1.2in]{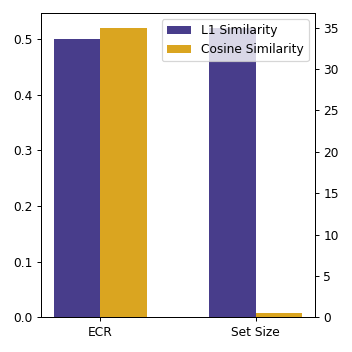}
        \caption{Similarity Measures}
    \end{subfigure}
    \caption{
    Ablation studies where (a) shows the benefits of pre-training the sentence transformer used in graph traversal and (b) shows a comparison when using L1 similarity as opposed to cosine similarity within \method. } 
    \label{fig:ab}
\end{figure}

\subsection{Effects of LM Pre-Training}

As shown in Figure \ref{fig:ab}(a), pre-training the sentence transformers (specifically for graph traversal) significantly reduces the size of the prediction set. This improvement is likely because fine-tuning at this stage enhances the quality of the initial traversal on the knowledge graph, leading to more accurate retrieval of answers. This process is analogous to better indexing in the context of Retrieval-Augmented Generation (RAG). This finding aligns with broader trends observed in graph-based question answering systems. However, when we pre-trained the large language model backbone, we did not observe a substantial improvement. This outcome aligns with previous findings suggesting that fine-tuning large language models can increase model uncertainty~\cite{ye2024benchmarking}. 

\subsection{Alternative Similarity Measures}

We also experimented with alternative methods for measuring textual similarity, specifically $S_1$ as defined in Eq. \eqref{eq:visiting}. As shown in Figure \ref{fig:ab}(b), we found that using cosine similarity significantly improves prediction efficiency while maintaining the user-defined error rate compared to L1 similarity. This improvement can be attributed to cosine similarity's focus on the angle between vectors, which makes it less sensitive to variations in vector magnitudes. This is particularly helpful in high-dimensional spaces, as it allows cosine similarity to more effectively capture semantic relationships between text representations.

\subsection{Ablation Studies Discussion}

Overall, \method demonstrates strong performance compared to existing KGQA methods even though it was not specifically designed for the traditional setting. Through the ablation studies, we observed that while \method benefits significantly from fine-tuning sentence transformers for graph traversal, the pre-training of large language models does not yield substantial gains, which we leave further investigation on finding improvements with fine-tuning LLMs as one future work in this direction. Lastly, through a case study we were able to more deeply understand the benefits of the proposed  trustworthy knowledge graph reasoning framework.

\section{Related Work}
\label{section:related}

\paragraph{Knowledge Graph Question Answering.}
Knowledge graph-based question answering (KGQA) has been a pivotal research topic in natural language processing due to its potential to leverage structured information for precise and interpretable answers. Traditional embedding-based methods transform entities and relations into a continuous embedding space and reason over these embeddings to produce answers~\cite{He_2021, skg, wang2021kepler}. Recently, the advent of large language models (LLMs) has shifted attention towards \textit{retrieval-augmented generation}, where the reasoning capabilities of LLMs are enhanced by retrieving relevant information from knowledge graphs. By retrieving structured data from the knowledge graph, LLMs can generate more accurate and contextually relevant answers. For instance, Luo et al. jointly fine-tuned a language model to generate both a relation plan and an answer by utilizing the retrieved reasoning paths~\cite{luo2024rog}. Similarly, Sun et al. and Wang et al. introduced novel LLM-guided graph traversal methods, enabling more efficient and effective navigation during the retrieval process~\cite{sun2023thinkongraph, wang2024knowledge}.

\paragraph{LLM Uncertainty Quantification.}
Uncertainty quantification has emerged as a critical research area for developing trustworthy AI systems, especially in contexts where decision-making involves significant risks. Recently, growing attention has been directed towards modeling uncertainty in large language models (LLMs). Conformal Language Models (CLM) extended the general risk control framework, allowing for the control of error rates in text generation while providing a more fine-grained level of confidence in the generated content~\cite{quach2023conformal}. LofreeCP further extended conformal prediction to black-box language models by approximating the non-conformity score with repeated prompting, thereby improving the reliability of LLM outputs~\cite{su2024api}. Additionally, researchers have proposed novel approaches to quantify uncertainty in retrieval-augmented generation by calibrating uncertainty across both the retrieval and generation processes with parameter tuning~\cite{rouzrokh2024conflare, li2024traq}.

\section{Conclusion}
\label{section:conclusion}

In this paper, we tackle the challenge of uncertainty quantification in knowledge graph question answering by integrating conformal prediction with KG-LLM models. Our architecture leverages the Learn Then Test (LTT) framework for multi-step calibration, delivering reliable results that meet pre-defined error rates while maintaining practical prediction set sizes. Extensive experiments show our method's effectiveness in balancing accuracy and uncertainty, making it suitable for real-world applications, especially those related to high stakes scenarios. 

While our focus is on knowledge graphs, this approach can be extend to open-domain QA, although challenges remain due to the lack of structured graph properties in such settings. Future work can explore these adaptations and ensure robust uncertainty quantification across various tasks.


\newpage

\bibliography{aaai25}

\appendix
\section{Appendix}

This technical appendix provides additional details and insights to complement the main paper. The rest of the appendix is organized as follows: Section A1 presents the notations used throughout the paper. Section A2 presents the pseudocode for \method. Section A3 includes the hardware specs and the hyper-parameter settings. Section A4 includes the quantitative experiment results of \method. Finally, Section A5 includes the Python implementation of \method.

\section{Notations}
We present the commonly used notations throughout the paper and this technical supplementary material in table \ref{tab:notations}.
\begin{table}[h!]
\centering 
\footnotesize 
\caption{List of Notations}
\begin{tabular}{cl}
\toprule
\textbf{Symbol} & \textbf{Description} \\ 
\midrule
$\mathcal{G}$  & Knowledge graph \\ 
$\mathcal{E}$            & Set of entities in the graph \\ 
$\mathcal{R}$            & Set of relations in the graph \\ 
$X_{test}$     & Input features for the test set \\ 
$D_{\text{cal}}$ & Calibration dataset \\
$Y_{test}$     & Ground-truth answers for the test set \\ 
$C_{\lambda}(X_{test})$ & Prediction set for $X_{test}$ under configuration $\lambda$\\ 
$\alpha$       & User-specified error rate for predictions \\ 
$\Lambda$ & Error rate configuration \\
$\delta$       & Confidence level for the error rate \\ 
$S$            & Non-conformal score function \\ 
$P$            & Probability measure \\ 
$\mathcal{P}$ & Set of retrieved reasoning paths\\
$\Phi$         & LLM generation function \\ 
$q_{S,D_{cal}}^{\alpha}$ & Quantile of the non-conformal scores \\ 

\bottomrule
\end{tabular}
\label{tab:notations}
\vspace{-2ex}
\end{table}

\begin{algorithm}[H]
\caption{\method Framework}
\label{alg:method}
\small  
\begin{algorithmic}[1] 
\Require Query $Q$, Question Entity $q$, Knowledge Graph $\mathcal{G}$, Error Rate $\alpha$, Calibrated Scores $q_{a_1}^{S_1, D_{cal}}, q_{a_2}^{S_1, D_{cal}}, q_{a_llm}^{S_1, D_{cal}}$
\Ensure Final Answer Set $\mathcal{A}$
\State Initialize Candidate Set $\mathcal{C} \gets \emptyset$
\State Initialize Search Queue $\mathcal{Q} \gets \{(q, Q)\}$

\While{$\mathcal{Q} \neq \emptyset$}
    \State Pop $(v, Q)$ from $\mathcal{Q}$
    \For{each $s \in \mathcal{N}(v)$} \Comment{Retrieve Candidate Path}
        \If{$S_1\Bigl(Q||(||_{i=0}^{j-1} r_i), r_j\Bigr) < q_{\alpha_1}^{S_1, \mathcal{D}_{\text{cal}}}$}\Comment{Eq. (4)}
        \State Add $s$ to Search Queue $\mathcal{Q}$
        \EndIf
    \EndFor

    \For{each $s \in \mathcal{N}(v)$} \Comment{Retrieve Candidate}
        \If{$S_1\Bigl(Q, ||_{i=0}^{j} r_i\Bigr) < q_{\alpha_2}^{S_1, \mathcal{D}_{\text{cal}}}$} \Comment{Eq. (5)}
            \State Add $s$ to Candidate Set $\mathcal{C}$
        \EndIf
    \EndFor
\EndWhile

\For{each $a \in \mathcal{C}$}
    \If{$S_1(a, \Phi(\mathcal{P})) < q_{\alpha_{\text{llm}}}^{S_1, \mathcal{D}_{\text{cal}}}$} \Comment{Eq. (6)}
        \State Add $a$ to Final Answer Set $\mathcal{A}$
    \EndIf
\EndFor

\State \Return $\mathcal{A}$

\end{algorithmic}
\end{algorithm}

\section{Pseudo-code}
We present the pseudo-code of \method in Algorithm \ref{alg:method}. We follow the notation and equations used throughout our paper. The algorithm begins by initializing an empty Candidate Set $\mathcal{C}$ to store potential answers and a Search Queue $\mathcal{Q}$ for the BFS(breadth-first-search). The BFS process is guided by the calibarated scores which serve as the threshold for adding neighboring nodes to $\mathcal{Q}$. 

Once all candidate nodes are collected, the algorithm filters them further by comparing the candidates with the llm generated answers $\Phi(\mathcal{P})$, ensuring that only those satisfying the condition in Equation (6) are included in the final answer set $\mathcal{A}$. Finally, the algorithm returns the set $\mathcal{A}$ as the output.

\section{Implementation Details}
\subsection{Computing Environment and Resources}
All of our experiments are conducted on a machine running Ubuntu 22.04.4 LTS with one GeForce RTX 4090 (24 GB VRAM) GPU and Intel i9 CPU with 64 GB RAM.  

\subsection{SentenceTransformers}
\paragraph{Model.} In Eqs. (4)-(6), 
\method's $S_1$ calculation requires first mapping the textual data into the embedding space. To this end, we employ SBERT as the encoder. Specifically, we use \texttt{all-MiniLM-L6-v2} model as the backbone model for pre-training. 

\paragraph{Fine-Tuning.} We leveraged fine-tuning to adapt the model to the dataset space. To ensure consistency with the traversal process, we fine-tuned the backbone SBERT using the \texttt{CosineSimilarityLoss}\footnote{https://sbert.net/docs/package\_reference/sentence\_transformer\\/losses.html\#cosinesimilarityloss}, where the input pairs are $\Bigl(Q||(||_{i=0}^{j-1} r_i), r_j\Bigr)$.

\paragraph{Hyper-parameters.} 
The key hyper-parameters for SBERT fine-tuning are the following:  \\
{\footnotesize 
\noindent\texttt{batch\_size: 32}\\
\texttt{scale: 20.0}\\
\texttt{similarity\_fct: 'cos\_sim'}\\
\texttt{epochs: 40}\\
\texttt{evaluation\_steps: 0}\\
\texttt{evaluator: NoneType}\\
\texttt{max\_grad\_norm: 1}\\
\texttt{optimizer\_class: AdamW}\\
\texttt{lr: 2e-05}\\
\texttt{scheduler: WarmupLinear}\\
\texttt{steps\_per\_epoch: null}\\
\texttt{warmup\_steps: 79636}\\
\texttt{weight\_decay: 0.01}
}

\subsection{\method}
\paragraph{Hyper-parameters.}For \method, the main hyper-parameter is $h$, which controls the search space of the parameters, and $\delta$, which controls the inner probability of the optimization target. In our experiments, we set $h$ to be different for the searc of $\alpha_1, \alpha_2$, and $\alpha_3$ to speed up the process. We set $h_1 = 0.3$, $h_2=0.3$, and $h_3=0.1$. In reality, smaller value would result in more fine-grained optimization. For $\delta$, we set the value to be 0.05.

Another critical consideration is the choice of Family-Wise Error Rate (FWER) control algorithm. Traditional methods, such as the Bonferroni correction, yield valid but often conservative prediction sets, leading to excessively broad predictions. To mitigate this, we adopted a heuristic approach where larger p-values are associated with better prediction coverage. This strategy allows us to effectively manage the search space frontier without significantly lowering the threshold (e.g., $\frac{\delta}{n}$). Specifically, by selecting the top p-value, we can optimize the size of the prediction set. In our experiments, we opted to select the top p-value, corresponding to setting $n=1$.

\paragraph{LLM.} The LLM backbone that we chose to use is \texttt{Llama-3-8b-instruct}. We used the recommended instruction prompt\footnote{https://llama.meta.com/docs/model-cards-and-prompt-formats/meta-llama-3/}. The sample prompt that we used is presented in Table~\ref{tab:case}. For generation, we use the following configurations with the LlamaForCausalLM and AutoTokenizer from Huggingface's Transformer class:\\

\noindent\texttt{temperature: 1}\\
\texttt{max\_new\_tokens: 10.0}\\
\texttt{num\_return\_sequence: 1}\\
\texttt{output\_scores: True}\\
\texttt{output\_logits: True}\\

For generation, we used $\text{batch\_size} = 1$ due to computational constraints. 

\subsection{Baselines}
\paragraph{Implementation.} For the WebQSP dataset, we take the experiment results from previous work~\cite{su2024api}. For CWQ, we implement the baselines using Python following their respective papers. We follow the same partition used in \method where the training partition is used for calibration.

\begin{table}[t]
\centering
\small
\caption{Prompt Examples.}
\label{tab:case}
\begin{tabular}{@{}p{0.25\linewidth}p{0.7\linewidth}@{}}
\toprule
\textbf{\method} & Based on the following reasoning paths: \{reasoning paths\}, \{question\}?  \\
\midrule 
\textbf{Example} & Based on the following reasoning paths: hinduism $\rightarrow$ \textit{practiced\_at\_location} $\rightarrow$ Indonesia. What are the religions practiced in Indonesia? \\
\midrule
\textbf{Answer}  & Hinduism \\
\bottomrule
\end{tabular}
\end{table}

\begin{table*}[t]
\centering
\caption{Experiment Results.}
\footnotesize
\label{tab:results}
\begin{tabular}{@{}ccccccccccccc@{}}
\toprule
$\alpha$ & \multicolumn{4}{c|}{0.1} & \multicolumn{2}{c}{0.15} &\multicolumn{2}{c|}{0.2}& \multicolumn{2}{c}{0.2}& \multicolumn{2}{c}{0.4}\\
  & \multicolumn{4}{c|}{WebQSP} & \multicolumn{2}{c}{WebQSP} & \multicolumn{2}{c|}{CWQ} & \multicolumn{2}{c}{WebQSP} & \multicolumn{2}{c}{CWQ} \\
 Metrics & \multicolumn{2}{c}{ECR} & \multicolumn{2}{c|}{SS} & ECR & SS & ECR & \multicolumn{1}{c|}{SS} & ECR & \multicolumn{1}{c}{SS} & ECR & SS \\ \midrule
Top K & - & - & - & - & - & - & - & - & - & - & - & - \\
Split CP & - & - & - & - & - & - & - & - & - & - & - & -  \\
CLM & - & - & - & - & - & - & - & - & - & - & - & - \\
LoFreeCP & - & - & - & - & - & - & - & - & - & - & - & - \\
$\method$ & \multicolumn{2}{c}{92.3} & \multicolumn{2}{c}{68.6} & 87.2 & 39.1 & 81.4 & 140 & 83.1 & 17.9 & 61.0 & 30.1 \\ \midrule
 $\alpha$ & \multicolumn{2}{c}{0.35} & \multicolumn{2}{c|}{0.5} & \multicolumn{2}{c}{0.4} & \multicolumn{2}{c|}{0.6}& \multicolumn{2}{c}{0.45} & \multicolumn{2}{c}{0.7} \\
  & \multicolumn{2}{c}{WebQSP} & \multicolumn{2}{c|}{CWQ} & \multicolumn{2}{c}{WebQSP} & \multicolumn{2}{c|}{CWQ} & \multicolumn{2}{c}{WebQSP} & \multicolumn{2}{c}{CWQ} \\
 Metrics & ECR & SS & ECR & \multicolumn{1}{c|}{SS} & ECR & SS & ECR & \multicolumn{1}{c|}{SS} & ECR & SS & ECR & SS \\ \midrule
Top K & 66.4 & 6.2 & 50.2 & 3.0 & 61.6 & 3.8 & 43.6 & 1.0 & 50.6 & 2.0 & 43.6 & 1.0 \\
Split CP & 65.3 & 4.5 & 64.1 & 3.0 & 61.6 & 3.0 & 64.1 & 3.0 & 50.2 & 1.4 & 44.9 & 1.0 \\
CLM & 65.3 & 4.5 & 42.4 & 3.5 & 60.5 & 2.8 & 42.4 & 1.4 & 50.1 & 1.3 & 30.2 & 0.9 \\
LoFreeCP & 65.1 & 5.3 & 43.3 & 5.3 & 60.0 & 2.7 & 43.3 & 4.2 & 50.3 & 1.0 & 28.4 & 3.0 \\ 
$\method$ & 70.1 & 2.6 & 50.5 & 3.2 & 63.1 & 0.9 & 45.2 & 2.6 & 52.9 & 0.5 & 29.1 & 0.7
\end{tabular}
\label{tab:webqsp}
\end{table*}

\section{Experiment Results}
In the main paper, we include the graphed results for \method for a cleaner
presentation. Here we provide the full result in Table \ref{tab:results}. We mark the entry ``-'' if the corresponding method is not able to produce a valid prediction set given the user-defined risk level $\alpha$. 

\section{Code}
Our implementation of \method is available for download in the google drive\footnote{https://drive.google.com/file/d/1vm2S1FUTSzMIAoI5R6-NX2JCHmIxkgxK/view?usp=sharing} for anonymous sharing. We will make the code public on GitHub upon publication. 

The code is structured as follows: 
\paragraph{qa.py} This is the entry point to the program. It will download the dataset if it is not present in the directory. Example: 
\texttt{python3 qa.py --epsilon 0.2 --llm\_ranker True --pre\_trained False --dataset\_name webqsp --calibrate\_rog True}. 

\paragraph{gen\_rule\_path.} Folder. This is the place where the planning rules are stored for retrieval. 

\paragraph{transformer\_models.} Folder that stores the weight of the fine-tuned SentenceTransformer model. The weights are loaded automatically when user executes the main program. 

\paragraph{model/cp.py}  Our implementation of \method. It includes some configurations (e.g. LLM based traversal strategies) that we experimented and they are safe to ignore. 

\paragraph{model/LLMs} Folder. Includes configurations that are safe to ignore. It is included in this version of release as qa.py explicitly refers to some of the classes when checking arguments. It will be cleaned up before the official release.


\end{document}